\documentclass{article}

\usepackage[accepted]{icml2018}

\usepackage{amsfonts} 
\usepackage{amsmath}
\usepackage{amssymb}
\usepackage{amsthm}
\usepackage{array}
\usepackage{booktabs}
\usepackage{color}
\usepackage{enumerate}
\usepackage{fancyhdr}
\usepackage{lineno,hyperref}
\usepackage{graphicx}
\usepackage{relsize}
\usepackage{microtype}
\usepackage{subfigure}
\usepackage{tikz}

\usepackage{enumitem}
\setlist{
topsep=-0.0em,   
itemsep=-0.0em, 
leftmargin=1.3pc, 
 }

\newtheorem{theorem}{Theorem}

\newtheorem{lemma}{Lemma}

\newtheorem{claim}{Claim}
\newcommand{\Tr}{\mathrm{Tr}}

\newcommand{\fro}{\mathrm{fro}}

\newcommand{\real}{\mathbb{R}}
\newcommand{\Loss}{\mathcal{L}}

\newcommand{\ww}{\mathbf{w}}

\newcommand{\xx}{\mathbf{x}}

\newcommand{\uu}{\mathbf{u}}

\newcommand{\yy}{\mathbf{y}}

\newcommand{\0}{\mathbf{0}}

\newcommand{\veps}{\varepsilon}
\newcommand{\R}{\mathbb{R}}

\newcommand{\M}{\mathbb{M}}

\DeclareMathOperator*{\argmin}{arg\,min}

\icmltitlerunning{Deep Linear Networks with Arbitrary Loss: All Local Minima Are Global}

\begin{document}

\twocolumn[
\icmltitle{Deep Linear Networks with Arbitrary Loss: All Local Minima Are Global}
\icmlsetsymbol{equal}{*}

\begin{icmlauthorlist}
\icmlauthor{Thomas Laurent}{equal,lmu}
\icmlauthor{James  H. von Brecht}{equal,lbs}
\end{icmlauthorlist}

\icmlaffiliation{lmu}{Department of Mathematics, Loyola Marymount University, Los Angeles, CA 90045, USA}
\icmlaffiliation{lbs}{Department of Mathematics and Statistics, California State University, Long Beach, Long Beach, CA 90840, USA}

\icmlcorrespondingauthor{Thomas Laurent}{tlaurent@lmu.edu}
\icmlcorrespondingauthor{James H. von Brecht}{james.vonbrecht@csulb.edu}


\vskip 0.3in
]

\printAffiliationsAndNotice{\icmlEqualContribution} 

\begin{abstract}
We consider deep linear networks with arbitrary convex differentiable  loss. We provide a short and elementary proof of the fact that all local minima are global minima if the hidden layers are either 1) at least as wide as the input layer, or 2) at least as wide as the output layer. This result is the strongest possible in the following sense: If the loss is convex and Lipschitz but not differentiable then deep linear networks \textbf{can} have sub-optimal local minima.
\end{abstract}

\section{Introduction}
Deep linear networks (DLN) are neural networks that have multiple hidden layers but have no nonlinearities between layers. That is, for given data points  $\{\xx^{(i)}\}_{i=1}^N$ the outputs of such networks are computed via a series
$$
\hat \yy^{(i)} = W_L  W_{L-1}\cdots W_1 \xx^{(i)}
$$
of matrix multiplications. Given a target $\yy^{(i)}$ for the $i^{\rm th}$ data point and a pairwise loss function $\ell(\hat \yy^{(i)}, \yy^{(i)}),$ forming the usual summation
\begin{equation} \label{loss}
\Loss(W_1,\ldots, W_L) =\frac{1}{N} \sum_{i=1}^{N} \ell(\hat \yy^{(i)}, \yy^{(i)})
\end{equation}
then yields the total loss.

Such networks have few direct applications, but they frequently appear as a class of toy models used to understand the loss surfaces of deep neural networks \cite{saxe2013exact,kawaguchi2016deep,  lu2017depth, hardt2016identity}. For example, numerical experiments indicate that DLNs exhibit some behavior that resembles the behavior of deep nonlinear networks during training \cite{saxe2013exact}. Results of this sort provide a small piece of evidence that DLNs can provide a decent simplified model of more realistic networks with nonlinearities.

From an analytical point-of-view, the simplicity of DLNs allows for a rigorous, in-depth study of their loss surfaces. These models typically employ a convex loss function $\ell(\hat \yy,\yy)$, and so with one layer (i.e. $L=1$) the loss $\Loss(W_1)$ is convex and the resulting optimization problem \eqref{loss} has no sub-optimal local minimizers.
With multiple layers (i.e. $L\ge 2$) the loss $\Loss(W_1, \ldots, W_L)$ is not longer convex, and so the question of paramount interest concerns whether this addition of depth and the subsequent loss of convexity creates sub-optimal local minimizers. Indeed, most analytical treatments of DLNs focus on this question.

We resolve this question in full for arbitrary convex differentiable loss functions. Specifically, we consider deep linear networks satisfying  the two following hypotheses:
\begin{enumerate}[label=(\roman*)]
\item The loss function $\hat \yy \mapsto \ell(\yy,\hat\yy)$ is convex and differentiable.
\item The thinnest layer is either the input layer or the output layer.
\end{enumerate}
Many networks of interest satisfy both of these hypotheses. The first hypothesis (i) holds for nearly all network criteria, such as the mean squared error loss, the logistic loss or the cross entropy loss, that appear in applications. In a classification scenario, the second hypothesis (ii) holds whenever each hidden layer has more neurons than the number of classes. Thus both hypotheses are often satisfied when using a deep linear network \eqref{loss} to model its nonlinear counterpart. In any such situation we resolve the deep linear problem in its entirety. 
\begin{theorem}
If hypotheses (i) and (ii) hold then \eqref{loss} has no sub-optimal minimizers, i.e. any local minimum is  global.
\end{theorem}
We provide a short, transparent proof of this result. It is easily accessible to any reader with a basic understanding of the singular value decomposition, and in particular, it does not rely on any sophisticated machinery from either optimization or linear algebra. Moreover, this theorem is the strongest possible in the following sense ---
\begin{theorem}
There exists a convex, Lipschitz but not differentiable function $\hat \yy \mapsto \ell(\yy,\hat\yy)$ for which \eqref{loss} has sub-optimal local minimizers.
\end{theorem}
In other words, we have a (perhaps surprising) hard limit on how far ``local equals global'' results can reach; differentiability of the loss is essential.

Many prior analytical treatments of DLNs focus on similar questions. For instance, both \cite{baldi1989neural} and \cite{baldi2012complex} consider ``deep" linear networks with two layers (i.e. $L=2$) and a mean squared error loss criterion. They provide a ``local equals global'' result under some relatively mild assumptions on the data and targets. More recently,  \cite{kawaguchi2016deep} proved that deep linear networks with arbitrary number of layers and with mean squared error loss do not have sub-optimal local minima under certain structural assumptions on the data and targets. The follow-up \cite{lu2017depth} futher simplifies the proof of this result and weakens the structural assumptions. Specifically, this result shows that the loss \eqref{loss} associated with a deep linear network has no sub-optimal local minima provided all of assumptions
\begin{enumerate}[label=(\roman*)]
 \item The loss $\ell(\hat \yy^{(i)}, \yy^{(i)})= \|\yy^{(i)} - \hat \yy^{(i)}\|^2$ is the mean squared error loss criterion;
 \item The data matrix $X=[\xx^{(1)}, \ldots, \xx^{(N)}]$  has full row rank;
  \item The target matrix $Y=[\yy^{(1)}, \ldots, \yy^{(N)}]$  has full row rank;
 \end{enumerate}
are satisfied. Compared to our result, \cite{lu2017depth} therefore allows for the hidden layers of the network to be thinner than the input and output layers. However, our result applies to network equipped with any differentiable convex loss (in fact any differentiable loss $\mathcal{L}$ for which first-order optimality implies global optimality) and we do not require any assumption on the data and targets. Our proof is also shorter and much more elementary by comparison.

\section{Proof of Theorem 1}
Theorem 1 follows as a simple consequence of a more general theorem concerning real-valued functions that take as input a product of matrices. That is, we view the deep linear problem as a specific instance of the following more general problem. Let  $\M_{m \times n }$ denote  the space of $m \times n$ real matrices, and let $f : \M_{d_{L} \times d_{0} } \to \R$ denote any differentiable function that takes $d_L \times d_0$ matrices as input. For any such function we may then consider both the single-layer optimization
\begin{equation*}
\qquad \text{(P1)}\;\,\left\{\quad \begin{aligned}
&\textrm{Minimize} \;\;\, f(A) \\
& \textrm{over all  $A $ in  $\M_{d_{L} \times d_{0} }$ }
\end{aligned}
\right.
\end{equation*}
as well as the analogous problem
\begin{equation*}
\text{(P2)}\;\,\left\{ \quad \begin{aligned}
& \textrm{Minimize} \;\;\,
 f(W_L W_{L-1} \cdots  W_1 ) \\
 & \textrm{over all $L$-tuples  $(W_1, \ldots, W_L)$} \\
 &  \textrm{in $ \M_{d_{1} \times d_{0} }   \times \cdots \times \M_{d_{L} \times d_{L-1}} $}
\end{aligned}
\hspace{0.5cm} \right.
\end{equation*}
that corresponds to a multi-layer deep linear optimization. In other words, in (P2) we consider the task of optimizing $f$ over those matrices $A \in \M_{d_{L} \times d_{0} }$ that can be realized by an $L$-fold product
\begin{equation}\label{eq:param}
A = W_{L}W_{L-1}\cdots W_{1} \qquad W_{\ell} \in \M_{d_{\ell} \times d_{\ell-1} }
\end{equation}
of matrices. We may then ask how the parametrization \eqref{eq:param} of $A$ as a product of matrices affects the minimization of $f,$ or in other words, whether the problems (P1) and (P2) have similar structure. At heart, the use of DLNs to model nonlinear neural networks centers around this question. 

Any notion of structural similarity between (P1) and (P2) should require that their global minima coincide. As a matrix of the form \eqref{eq:param} has rank at most $\min\{d_0,\ldots,d_{L}\}$, we must impose the structural requirement
\begin{equation} \label{struct}
\min\{d_1,\ldots,d_{L-1}\} \geq \min\{d_{L},d_{0}\}
\end{equation}
in order to guarantee that \eqref{eq:param} does, in fact, generate the full space of $d_{L} \times d_{0}$ matrices. Under this assumption we shall prove the following quite general theorem.

\begin{theorem}\label{main_theorem}
Assume that $f(A)$ is any differentiable function and that the structural condition \eqref{struct}
holds. Then at \textbf{any} local minimizer $\big(\hat W_{1}, \ldots , \hat W_{L}\big)$ of (P2) the optimality condition
$$
\nabla f\big( \hat A \big) = 0 \qquad \qquad \hat A := \hat W_{L}\hat W_{L-1}\cdots \hat W_{1}
$$
is satisfied.
\end{theorem}
Theorem 1 follows as a simple consequence of this theorem. The first hypothesis (i) of theorem 1 shows that the total loss \eqref{loss} takes the form
$$
\Loss(W_1,\ldots, W_L)= f(W_L \cdots W_1 )
$$
for $f(A)$ some  convex and differentiable function. The structural hypothesis \eqref{struct} is equivalent to the second hypothesis (ii) of theorem 1, and so we can directly apply theorem \ref{main_theorem} to conclude that a local minimum $\big(\hat W_{1}, \ldots , \hat W_{L}\big)$ of $\Loss$ corresponds to a critical point $\hat A = \hat W_{L} \cdots \hat W_{1}$ of $f(A)$, and since $f(A)$ is convex, this critical point is necessarily a global minimum. 

Before turning to the proof of theorem \ref{main_theorem} we recall a bit of notation and provide a calculus lemma. Let
\begin{align*}
\langle A , B \rangle_\fro &:= \Tr (A^TB)  = \sum_i \sum_j A_{ij} B_{ij}  \quad \text{and} \\
\|A\|_\fro^2 &:= \langle A , A \rangle_\fro
\end{align*}
denote the Frobenius dot product and the Frobenius norm, respectively. Also, recall that for a differentiable function $\phi: \M_{m \times n}  \mapsto \real$ its gradient $\nabla \phi(A) \in \M_{m \times n}$ is the unique matrix so that the equality
\begin{equation} \label{derivative_def}
\phi(A+H)= \phi(A) + \langle \nabla \phi(A) , H \rangle_\fro + o\left( \|H\|_\fro \right)
\end{equation}
holds. If $F(W_1,\ldots,W_L) := f(W_L\cdots W_1)$ denotes the objective of interest in (P2) the following lemma gives the partial derivatives of $F$ as a function of its arguments. 
\begin{lemma} \label{derivatives}
The partial derivatives of $F$ are given by
 \begin{align*}
&\nabla_{W_1} F(W_1, \ldots, W_L)=  W^{T}_{2,+}\nabla f\big(  A\big), \\
&\nabla_{W_k} F(W_1, \ldots, W_L) = W^T_{k+1,+} \nabla f(A)  W^T_{k-1,-}, \\
&\nabla_{W_L} F(W_1, \ldots, W_L) =\nabla f\big(  A\big)  W^{T}_{L-1,-},
\end{align*}
where $A$ stands for the full product $A:=W_L \cdots W_1$ and $W_{k,+},W_{k,-}$ are the truncated products 
\begin{align}
 W_{k,+} &:=  W_{L} \cdots  W_{k}, \nonumber \\ 
 W_{k,-} &:=  W_{k}\cdots W_{1}  \label{notation22}.
\end{align}
\end{lemma}
\begin{proof}
The definition \eqref{derivative_def} implies
\begin{align*}
& F(W_1,\ldots,W_{k-1},W_k + H,W_{k+1},\ldots,W_{L})  \\
&=f\big(A + W_{k+1,+} H W_{k-1,-}\big)  \\
&=f(A) + \langle \nabla f(A) , W_{k+1,+} H W_{k-1,-} \rangle_{{\rm fro}} + o\big( \|H\|_{ {\rm fro}}\big).
\end{align*}
Using  the cyclic property $\Tr(ABC) = \Tr(CAB)$  of the trace then gives
\begin{align*}
\langle \nabla  f(A) \; &, \;  W_{k+1,+} H W_{k-1,-} \; \rangle_{{\rm fro}}  \\
& = \Tr\left(  \; \nabla f(A)^T  W_{k+1,+} H W_{k-1,-} \; \right) \\
& = \Tr\left(  \;  W_{k-1,-} \nabla f(A)^T  W_{k+1,+} H \; \right) \\
& = \langle \;  W^{T}_{k+1,+}\nabla f(A)W^{T}_{k-1,-}  \; ,  \;H \;\rangle_{{\rm fro}}
\end{align*}
which,  in light of \eqref{derivative_def}, gives the desired formula for $\nabla_{W_k} F$. The formulas for $\nabla_{W_1} F$ and $\nabla_{W_L} F$ are obtained similarly.
\end{proof}

\begin{proof}[{\bf Proof of Theorem 3:}]
To prove theorem \ref{main_theorem} it suffices to assume that $d_{L} \geq d_{0}$ without loss of generality. This follows from the simple observation that 
$$
g\big( A \big) := f\big( A^{T} \big)
$$
defines a differentiable function of $d_{0} \times d_{L}$ matrices for $f(A)$ any differentiable function of $d_{L} \times d_{0}$ matrices. As a point $\big(W_{1},\ldots,W_{L}\big)$ defines a local minimum of $f\big( W_{L} W_{L-1}\cdots W_{1} \big)$ if and only if $\big(W^{T}_{1}, \ldots, W^{T}_{L}\big)$ defines a minimum of $g\big( V_{1} \cdots V_{L-1}  V_{L} \big),$ the theorem for the case $d_{L} < d_{0}$ follows by appealing to its $d_{L} \geq d_{0}$ instance. It therefore suffices to assume that $d_{L} \geq d_{0},$ and by the structural assumption that $d_{k} \geq d_0$, throughout the remainder of the proof.

Consider any local minimizer $\big(\hat W_{1},\ldots,\hat W_{L}\big)$ of $F$ and denote by $\hat A$, $\hat W_{k,+}$ and $\hat W_{k,-}$  the corresponding full and truncated products (c.f. \eqref{notation22}). By definition of a local minimizer there exists some $\veps_0 > 0$ so that
\begin{equation} \label{loc_min}
F(W_1, \ldots, W_L) \geq F(\hat W_1, \ldots, \hat W_L) =  f\big( \hat A \big)
\end{equation}
whenever the family of inequalities 
$$\| {W}_{\ell} - \hat W_{\ell}\|_{ \mathrm{fro} } \leq \veps_0 \quad \text{for all} \quad 1 \leq \ell \leq L$$
all hold.  Moreover, lemma \ref{derivatives} yields
\begin{align}\label{eq:foopt}
(\mathrm{i})& \;\; 0 = \hat W^{T}_{2,+}\nabla f\big( \hat A\big), \nonumber \\
(\mathrm{ii})& \;\; 0 = \hat W^{T}_{k+1,+}\nabla f\big( \hat A\big)\hat W^{T}_{k-1,-} \quad \forall \; 2 \leq k \leq L-1, \nonumber \\
(\mathrm{iii})& \;\; 0 = \nabla f\big( \hat A\big) \hat W^{T}_{L-1,-}.
\end{align}
since all partial derivatives must vanish at a local minimum. If $\hat{W}_{L-1,-}$ has a trivial kernel, i.e. $\mathrm{ker}(\hat{W}_{L-1,-})=\{\0\}$, then the theorem follows easily. The critical point condition \eqref{eq:foopt} part (iii) implies 
$$\hat{W}_{L-1,-}\nabla f\big( \hat A \big )^{T} = 0,$$
and since $\hat{W}_{L-1,-}$ has a trivial kernel this implies $\nabla f\big( \hat{A} \big) =\nabla f\big( \hat W_{L}\hat W_{L-1}\cdots \hat W_{1} \big) = 0$ as desired.

The remainder of the proof concerns the case that $\hat{W}_{L-1,-}$ has a nontrivial kernel. The main idea  is to use this nontrivial kernel to construct a family of infinitesimal perturbations of the local minimizer  $\big(\hat W_{1},\ldots, \hat W_{L}\big)$ that leaves the overall product $\hat W_L \cdots \hat W_1$ unchanged. In other words, the family of perturbations $\big(\tilde W_{1},\ldots, \tilde W_{L}\big)$ satisfy
\begin{align}
&\| \tilde W_\ell - \hat W_\ell \|_\fro \le \epsilon_0/2  \quad \forall \ell = 1, \ldots, L  \label{zizi1}, \\
&\tilde W_L \tilde W_{L-1} \cdots \tilde W_1 = \hat W_L \hat W_{L-1} \cdots \hat W_1 \label{zizi2}.
\end{align}
Any such perturbation also defines a local minimizer.
\begin{claim}\label{claim:local}
Any tuple of matrices $\big(\tilde W_{1},\ldots, \tilde W_{L}\big)$ satisfying \eqref{zizi1} and \eqref{zizi2} is necessarily a local minimizer $F$.
\end{claim}
\begin{proof}
 For any matrix $W_{\ell}$ satisfying $\|W_{\ell} - \tilde W_{\ell}\|_{ \mathrm{fro} } \leq \veps_0/2$, inequality \eqref{zizi1} implies that
\begin{align*}
\|W_{\ell} - \hat W_{\ell}\|_{ \mathrm{fro} } &\leq \|W_{\ell} - \tilde W_{\ell}\|_{ \mathrm{fro} } + \|\tilde W_{\ell} - \hat W_{\ell}\|_{ \mathrm{fro} }  \leq \veps_0
\end{align*}
Equality \eqref{zizi2} combined to 
  \eqref{loc_min} then leads to
\begin{multline*}
F\big( W_{1}, \ldots, W_{L}\big) \geq f\big( \hat A \big) = f(\hat W_L \cdots \hat W_1 ) \\
= f(\tilde W_L \cdots \tilde W_1 ) = F\big( \tilde{W}_1, \ldots, \tilde{W}_L \big)
\end{multline*}
for any $W_{\ell}$ with $\|W_{\ell} - \tilde W_{\ell}\|_{ \mathrm{fro} } \leq \veps_0/2$ and so the point $(\tilde W_{1},\ldots,\tilde W_L)$ defines a local minimum. 
\end{proof}

The construction of such perturbations requires a preliminary observation and then an appeal to the singular value decomposition. Due to the definition of $\hat W_{k,-}$ it follows that $\ker(\hat W_{k+1,-}) = \ker( \hat W_{k+1} \hat W_{k,-}) \supseteq \ker( \hat W_{k,-} ),$ and so the chain of inclusions
\begin{equation} \label{lala}
\mathrm{ker}(\hat{W}_{1,-}) \subseteq \mathrm{ker}(\hat{W}_{2,-}) \subseteq \cdots \subseteq \mathrm{ker}(\hat{W}_{L-1,-})
\end{equation}
holds. Since $\hat{W}_{L-1,-}$ has a  nontrivial kernel, the chain of inclusions \eqref{lala} implies that there exists
 $k_* \in \{1, \ldots, L-1\}$ such that
\begin{align}\label{eq:k*}
 &\mathrm{ker}(\hat{W}_{k,-})=\{\0\}  \quad \text{if } k < k_* \\
 &\mathrm{ker}(\hat{W}_{k,-})\neq \{\0\} \quad   \text{if } k \ge k_*  \label{eq:kk}
\end{align}
In other words, $\hat{W}_{k_*,-}$ is the first matrix appearing in \eqref{lala} that has a nontrivial kernel.

The structural requirement \eqref{struct} and the assumption that $d_{L} \geq d_{0}$ imply that $d_k \ge d_0$ for all $k$, and so the matrix $\hat W_{k,-} \in \M_{d_k \times d_0}$ has more rows than columns. As a consequence its full singular  value decomposition
\begin{equation}\label{eq:svd}
\hat W_{k,-} = \hat U_k \hat \Sigma_k \hat V^{T}_k
\end{equation}
has the shape depicted in figure \ref{fig:svd}. The matrices $\hat U_k \in \M_{d_k \times d_k}$ and $\hat V_k \in \M_{d_0 \times d_0}$ are   orthogonal, whereas $\hat \Sigma_{k} \in \M_{d_k \times d_0}$ is a diagonal matrix containing the singular values of $\hat W_{k,-}$ in descending order.  From \eqref{eq:kk} $\hat W_{k,-}$ has a nontrivial kernel for all $k\geq k_*,$ and so in particular its least singular value is zero. In particular, the $(d_0,d_0)$ entry of $\hat \Sigma_k$ vanishes if $k \ge k_*$. Let  $\hat \uu_{k}$ denote the corresponding $d_0^{ {\rm th}}$ column of $\hat U_k$, which exists since $d_k \ge d_0$.

\begin{claim}\label{claim:pert}
Let $\ww_{k_*+1},\ldots,\ww_{L}$ denote any collection of vectors and  $\delta_{k_*+1},\ldots,\delta_{L}$  any collection of scalars satisfying
\begin{align}
&\ww_k \in \real^{d_k}, \quad \|\ww_k\|_2=1 \quad \text{and} \label{caca1}  \\ 
& 0 \leq \delta_{k} \leq \epsilon_0/2 \label{caca2}  
\end{align}
for all $k_*+1 \le k \le L$. Then the tuple of matrices $(\tilde W_1,\ldots, \tilde W_L)$ defined by
\begin{equation}\label{eq:newpts}
\tilde{W}_{k} := \begin{cases}
\hat W_{k} & \text{if} \quad 1 \leq k \leq k_*\\
\hat W_{k} + \delta_{k} \ww_{k} \hat \uu^{T}_{k-1}& \text{else},
\end{cases}
\end{equation}
satisfies \eqref{zizi1} and $\eqref{zizi2}$.
\end{claim}

 \begin{figure}[t]
\begin{center}
\begin{tikzpicture}
\node [above] at (0.5,2.4) {\scriptsize{$d_0$}};
\draw[<->] (0,2.4) -- (1,2.4);
\draw[<->] (-0.4,0.5) -- (-0.4,2);
\node [left] at (-0.4,1.25) {\scriptsize{$d_k$}};
\draw[fill=gray!30!white] (0,0.5) rectangle (1,2); 
\node [above] at (1.5,1) {$=$};
\draw[fill=gray!30!white] (2,0.5) rectangle (3.5,2); 
\draw (4.3-0.5,0.5) rectangle (5.3-0.5,2); 
\fill[gray!30!white] (4.3-0.5,1.9) -- (4.3-0.5,2) -- (4.4-0.5,2)    --   (5.3-0.5,1.1)  -- (5.3-0.5,1) -- (5.2-0.5,1)   -- cycle;
\node [red,above] at (5.2-0.5,0.9) {\scriptsize{$0$}};
\draw[fill=gray!30!white] (5.6-0.5,1) rectangle (6.6-0.5,2); 
\node [below] at (0.5,-0.2+0.5) {$\hat W_{k,-}$};
\node [below] at (3,-0.2+0.5) {$\hat U_{k}$};
\node [below] at (4.8-0.5,-0.2+0.5) {$\hat \Sigma_{k}$};
\node [below] at (6.1-0.5,-0.2+0.5) {$\hat V_{k}^T$};
\draw[red,thick](3,2)--(3,1.4) ;
\draw[red,thick](3,1.05)--(3,0.5) ;
\node [red,below] at (3,1.45) {\scriptsize{$\hat  {\bf u}_{k}$}};
 \node[anchor=east] at (3.12,2) (text) {};
  \node[anchor=west, red] at (3.3,2.7) (description) {\scriptsize{$d_0^{ {\rm th}}$ column}};
  \draw[->,red] (description) edge[out=180,in=90,->] (text);
\end{tikzpicture}
\end{center}
\caption{Full singular value decomposition of $\hat W_{k,-} \in \M_{d_k \times d_0}$. If $k \ge k_*$ then $\hat W_{k,-}$ does not have full rank and so the $(d_0,d_0)$ entry of $\hat \Sigma_k$ is $0$. The  $d_0^{{\rm th}}$ column of $\hat U_k$ exists since $d_k \ge d_0$.}
\label{fig:svd}
\end{figure}
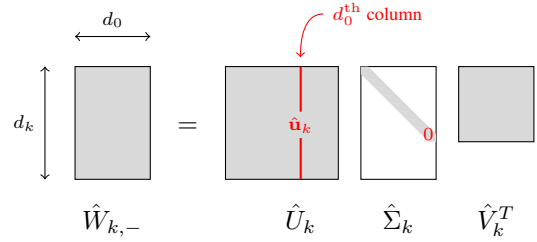

\begin{proof}
Inequality \eqref{zizi1} follows from the fact that
$$\| \tilde W_k - \hat W_k \|_\fro = \delta_{k} \| \ww_{k} \hat \uu^{T}_{k-1} \|_\fro =  \delta_{k}  \| \ww_{k} \|_2  \|  \hat \uu_{k-1} \|_2$$
for all  $k> k^*$, together with the fact that $\hat \uu_{k-1} $ and $\ww_{k}$ are unit vectors and that $0 \le \delta_{k} \le \epsilon_0/2$.
 
 To prove \eqref{zizi2} let
$\tilde W_{k,-}  =\tilde W_k \cdots \tilde W_1$ and  $  \hat W_{k,-}=  \hat W_k \cdots \hat W_1$ denote the truncated products of the matrices $\tilde W_k$ and $\hat W_k$. The equality $\tilde W_{k_*,-}=   \hat W_{k_*,-}$ is immediate from the definition \eqref{eq:newpts}. The equality \eqref{zizi2} will then follow from showing that
\begin{equation}\label{eq:inducc}
\tilde W_{k,-}=   \hat W_{k,-} \quad \text{for all} \quad k_* \le k \le L.
\end{equation}
Proceeding by induction, assume that $\tilde W_{k,-}=   \hat W_{k,-}$ for a given $k \ge k_*$. Then
\begin{align*}
\tilde W_{k+1,-}&= \tilde W_{k+1} \tilde W_{k,-} \\
&=\tilde W_{k+1} \hat W_{k,-}  \quad \text{(induction hypothesis)}\\
&=  \left( \hat W_{k+1} + \delta_{k+1} \ww_{k+1} \hat \uu^{T}_{k} \right) \hat W_{k,-}  \\
&=   \hat W_{k+1,-} + \delta_{k+1} \ww_{k+1}  \uu_{k}^T  \hat W_{k,-}   
\end{align*}
The second term of the last line vanishes, since
$$
\uu_{k}^T  \hat W_{k,-}=  \uu_{k}^T\hat U_k \hat \Sigma_k \hat V^{T}_k = 
\mathbf{e}^{T}_{d_0} \hat \Sigma_{k} \hat V^{T}_{k}=\0
$$
with $\mathbf{e}_{d_0} \in \R^{d_{k}}$ the $d_0^{ {\rm th}}$ standard basis vector. The second equality comes from the fact that the columns of $\hat U_{k}$ are orthonormal, and  the last equality comes from the fact that $\mathbf{e}_{d_0}^T \Sigma_{k_*}=\0$  since the $d_0^{{\rm th}}$ row of $\hat \Sigma_{k_*}$ vanishes. Thus \eqref{eq:inducc} holds, and so \eqref{zizi2} holds as well.
\end{proof}

Claims \ref{claim:local} and  claim \ref{claim:pert} show that the perturbation  $\big(\tilde W_{1},\ldots,\tilde W_{L}\big)$ defined  by \eqref{eq:newpts} is  a local minimizer of $F$. The critical point conditions
\begin{align*}
(\mathrm{i})& \;\; 0 = \tilde W^{T}_{2,+}\nabla f\big( \tilde A\big), \nonumber \\
(\mathrm{ii})& \;\; 0 = \tilde W^{T}_{k+1,+}\nabla f\big( \tilde A\big)\tilde W^{T}_{k-1,-} \quad \forall \; 2 \leq k \leq L-1, \nonumber \\
(\mathrm{iii})& \;\; 0 = \nabla f\big( \tilde A\big) \tilde W^{T}_{L-1,-}
\end{align*}
therefore hold as well {\bf for all choices of $\ww_{k_*+1},\ldots,\ww_{L}$  and  $\delta_{k_*+1},\ldots,\delta_{L}$ satisfying \eqref{caca1} and \eqref{caca2}.}

The proof concludes by appealing to this family of critical point relations. If $k_*>1$ the transpose of  condition (ii) gives
\begin{equation} \label{aaa}
  \hat W_{k_*-1,-} \nabla f\big( \hat A\big)^T \tilde W_{k_*+1,+}=0
\end{equation}
since the equalities $\tilde W_{k_*-1,-} =  \hat W_{k_*-1,-}$ (c.f. \eqref{eq:newpts}) and $\tilde A= \tilde W_L \cdots \tilde W_1 = \hat W_L \cdots \hat W_1 = \hat A$ (c.f. \eqref{zizi2}) both hold. But $\mathrm{ker}(\hat W_{k_*-1,-})= \{\0\}$ by definition of $k_*$ (c.f. \eqref{eq:k*}), and so
\begin{equation} \label{bbb}
\nabla f \big( \hat A \big)^{T}\tilde W_{L}\cdots \tilde W_{k_*+1} = 0.
\end{equation}
must hold as well. If $k_*=1$ then  \eqref{bbb} follows trivially from the critical point condition (i).  Thus \eqref{bbb} holds for all choices of $\ww_{k_*+1},\ldots,\ww_{L}$  and  $\delta_{k_*+1},\ldots,\delta_{L}$ satisfying \eqref{caca1} and \eqref{caca2}. First choose $\delta_{k_*+1} = 0$ so that  $\tilde W_{k_*+1} =  \hat W_{k_*+1}$ and apply   \eqref{bbb} to find
\begin{equation}
\label{ccc}
\nabla f \big( \hat A \big)^{T}\tilde W_{L}\cdots \tilde W_{k_*+2}\hat W_{k_*+1} = 0.
\end{equation}
Then take any $\delta_{k_*+1} >0$ and substract \eqref{ccc} from \eqref{bbb} to get
\begin{align*}
&\frac{1}{\delta_{k_*+1}}
\nabla f \big( \hat A \big)^{T}\tilde W_{L}\cdots  \tilde W_{k_*+2} \left( \tilde W_{k_*+1} -\hat W_{k_*+1} \right)\\
&=\nabla f \big( \hat A \big)^{T}\tilde W_{L}\cdots \tilde W_{k_*+2}\big( \ww_{k_*+1} \hat \uu^{T}_{k_*} \big)  
= 0
\end{align*}
for $\ww_{k_*+1}$ an arbitrary vector with unit length. Right multiplying the last equality by $\hat \uu_{k_*}$ and using the fact that $ (\ww_{k_*+1} \hat \uu^{T}_{k_*})  \hat \uu_{k_*}= \ww_{k_*+1} \hat \uu_{k_*}^T \hat \uu_{k_*}=\ww_{k_*+1} $ shows
\begin{equation} \label{ddd}
\nabla f \big( \hat A \big)^{T}\tilde W_{L}\cdots \tilde W_{k_*+2}\ww_{k_*+1} = 0
\end{equation}
for all choices of $\ww_{k_*+1}$ with unit length. Thus \eqref{ddd} implies
$$
\nabla f \big( \hat A \big)^{T}\tilde W_{L}\cdots \tilde W_{k_*+2} = 0
$$
for all choices of $\ww_{k_*+2},\ldots,\ww_{L}$  and  $\delta_{k_*+2},\ldots,\delta_{L}$ satisfying \eqref{caca1} and \eqref{caca2}. The claim
$$
\nabla f \big( \hat A \big) = 0
$$
then follows by induction. 
\end{proof}

\section{Concluding Remarks}
Theorem 3 provides the mathematical basis for our analysis of deep linear problems. We therefore conclude by discussing its limits.

First, theorem 3 fails if we refer to critical points rather than local minimizers. To see this, it suffices to observe that the critical point conditions for problem (P2),
\begin{align*}
(\mathrm{i})& \;\; 0 = \hat W^{T}_{2,+}\nabla f\big( \hat A\big), \nonumber \\
(\mathrm{ii})& \;\; 0 = \hat W^{T}_{k+1,+}\nabla f\big( \hat A\big)\hat W^{T}_{k-1,-} \quad \forall \; 2 \leq k \leq L-1, \nonumber \\
(\mathrm{iii})& \;\; 0 = \nabla f\big( \hat A\big) \hat W^{T}_{L-1,-}
\end{align*}
where $\hat W_{k,+} := \hat W_{L} \cdots \hat W_{k+1}$ and $\hat W_{k,-} := \hat W_{k-1}\cdots \hat W_{1} $, clearly hold if $L \geq 3$ and all of the $\hat W_{\ell}$ vanish. In other words, the collection of zero matrices always defines a critical point for (P2) but clearly $\nabla f\big( \mathbf{0} \big)$ need not vanish. To put it otherwise, if $L \geq 3$ the problem (P2) always has saddle-points even though all local optima are global.

Second, the assumption that $f(A)$ is differentiable is necessary as well. More specifically, a function  of the form 
$$
F(W_1,\ldots,W_L) := f( W_{L}\cdots W_{1} )
$$
can have sub-optimal local minima if $f(A)$ is convex and globally Lipschitz but is not differentiable. A simple example demonstrates this, and therefore proves theorem 2. For instance, consider the  bi-variate convex function
\begin{equation}\label{eq:thiscex}
f(x,y) := |x| + (1-y)_{+} - 1, \quad (y)_{+} := \max\{y,0\}, 
\end{equation}
which is clearly globally Lipschitz but not differentiable. The set
$$
\argmin \; f := \{ (x,y) \in \R^{2} : x=0, \,y \geq 1\}
$$
furnishes its global minimizers while $f_{ {\rm opt} } = -1$ gives the optimal value. For this function even a two layer   deep linear problem
$$
F\big( W_1 , W_2) := f( W_2 W_1) \quad W_2 \in \R^{2}, \; W_1 \in \R
$$
has sub-optimal local minimizers; the point
\begin{equation}\label{eq:thisthing}
(\hat W_1 , \hat W_2) = \left( 0 , \begin{bmatrix} 1 \\ 0\end{bmatrix} \right)
\end{equation}
provides an example of a sub-optimal solution. The set of all possible points in $\R^{2}$
\begin{align*}
&\mathcal{N}(\hat W_1,\hat W_2) := \\
 &\left\{ W_2W_1  : \| W_2 - \hat W_2\| \leq \frac1{4} , \; \| W_1 - \hat W_1 \| \leq \frac1{4} \right\}
\end{align*}
generated by a $1/4$-neighborhood of the optimum \eqref{eq:thisthing} lies in the two-sided, truncated cone
$$
\mathcal{N}(\hat W_1,\hat W_2) \subset \left\{ (x,y) \in \R^{2} : |x| \leq \frac{1}{2} , |y| \leq \frac1{2}|x| \right\},
$$
and so if we let $x \in \R$ denote the first component of the product $W_2 W_1$ then the inequality
$$
f(W_2 W_1) \geq \frac1{2}|x| \geq 0 = f( \hat W_2 \hat W_1 )
$$
holds on $\mathcal{N}(\hat W_1,\hat W_2)$ and so $(\hat W_1,\hat W_2)$ is a sub-optimal local minimizer. Moreover, the minimizer $(\hat W_1,\hat W_2)$ is a \emph{strict} local minimizer in the only sense in which strict optimality can hold for a deep linear problem. Specifically, the strict inequality 
\begin{equation}\label{eq:thisargument}
f(W_2 W_1) > f( \hat W_2 \hat W_1 )
\end{equation}
holds on $\mathcal{N}(\hat W_1,\hat W_2)$ unless $W_2 W_1 = \hat W_2 \hat W_1 = \0;$ in the latter case $(W_1,W_2)$ and $(\hat W_1,\hat W_2)$ parametrize the same point and so their objectives must coincide. We may identify the underlying issue easily. The proof of theorem \ref{main_theorem} requires a single-valued derivative $\nabla f(\hat A)$ at a local optimum, but with $f(x,y)$ as in \eqref{eq:thiscex} its subdifferential
$$
\partial f(\0) = \{ (x,y) \in \R^{2} : -1 \leq x \leq 1,\, y=0 \}
$$ 
is multi-valued at the sub-optimal local minimum \eqref{eq:thisthing}. In other words, if a globally convex function $f(A)$ induces sub-optimal local minima in the corresponding deep linear problem then $\nabla f( \hat A)$ cannot exist at any such sub-optimal solution (assuming the structural condition, of course).

Third, the structural hypothesis 
$$d_\ell \geq \min\{d_{L},d_{0}\}  \quad \text{for all } \ell \in \{1, \ldots, L \}$$ is necessary for theorem \ref{main_theorem} to hold as well. If $d_{\ell} < \min\{ d_0,d_{L}\}$ for some $\ell$ the parametrization 
$$
A = W_{L}\cdots W_{1}
$$
cannot recover full rank matrices. Let $f(A)$ denote any function where $\nabla f $ vanishes only at full rank matrices. Then 
$$
\nabla f\big ( W_{L}\cdots W_{1}\big) \neq \0
$$
at all critical points of (P2), and so theorem \ref{main_theorem} fails.

Finally, if we do not require convexity of $f(A)$ then it is not true, in general, that local minima of (P2) correspond to minima of the original problem. The functions
$$
f(x,y) = x^{2} - y^2 \qquad F(W_1,W_2) = f(W_2 W_1)
$$
and the minimizer \eqref{eq:thisthing} illustrate this point. While the origin is clearly a saddle point of the one layer problem, the argument leading to \eqref{eq:thisargument} shows that \eqref{eq:thisthing} is a local minimizer for the deep linear problem. So in the absence of additional structural assumptions on $f(A),$ we may infer that a minimizer of the deep linear problem satisfies first-order optimality for the original problem, but nothing more.

\bibliography{bibliography}
\bibliographystyle{icml2018}

\end{document}